\documentclass{article}
\usepackage{ijcai18}

\usepackage{times}
\usepackage{xcolor}
\usepackage{soul}
\usepackage[utf8]{inputenc}
\usepackage[small]{caption}

\usepackage{epsfig}
\usepackage{graphicx}
\usepackage{amsmath}
\usepackage{amssymb}
\usepackage{amsthm}
\usepackage{multirow}

\usepackage{algorithm}
\usepackage{algorithmic}
\usepackage{caption}
\usepackage{subcaption}

\newtheorem{proposition}{Proposition}
\newtheorem{theorem}[proposition]{Theorem}

\title{Generative Adversarial Positive-Unlabelled Learning}

\author{
Ming Hou$^1$,
Brahim Chaib-draa$^2$,
Chao Li$^1$,
Qibin Zhao$^1$,
\\
$^1$ Center for Advanced Intelligence Project, RIKEN, Tokyo, Japan \\
$^2$ Department of Computer Science and Software Engineering, Laval University, Quebec, Canada \\
ming.hou@riken.jp,
brahim.hou@riken.jp,
chao.li.hf@riken.jp,
qibin.zhao@riken.jp
}

\begin{document}
\maketitle
\begin{abstract}
    In this work, we consider the task of classifying binary positive-unlabeled (PU) data. The existing discriminative learning based PU models attempt to seek an optimal reweighting strategy for U data, so that a decent decision boundary can be found.
    However, given limited P data, the conventional PU models tend to suffer from overfitting when adapted to very flexible deep neural networks. In contrast, we are the first to innovate a totally new paradigm to attack the binary PU task, from perspective of generative learning by leveraging the powerful generative adversarial networks (GAN). Our generative positive-unlabeled (GenPU) framework incorporates an array of discriminators and generators that are endowed with different roles in simultaneously producing positive and negative realistic samples. We provide theoretical analysis to justify that, at equilibrium, GenPU is capable of recovering both positive and negative data distributions. Moreover, we show GenPU is generalizable and closely related to the semi-supervised classification. Given rather limited P data, experiments on both synthetic and real-world dataset demonstrate the effectiveness of our proposed framework. With infinite realistic and diverse sample streams generated from GenPU, a very flexible classifier can then be trained using deep neural networks.
\end{abstract}

\section{Introduction}
Positive-unlabeled (PU) classification \cite{denis1998pac,denis2005learning} has gained great popularity in dealing with limited partially labeled data and succeeded in a broad range of applications such as automatic label identification. Yet, PU can be used for the detection of outliers in an unlabeled dataset with knowledge only from a collection of inlier data \cite{hido2008inlier,smola2009relative}. PU also finds its usefulness in `one-vs-rest' classification task such as land-cover classification (urban vs non-urban) where non-urban data are too diverse to be labeled than urban data \cite{li2011positive}.

The most commonly used PU approaches for binary classification can typically be categorized, in terms of the way of handling U data, into two types \cite{kiryo2017positive}. One type such as \cite{liu2002partially,li2003learning} attempts to recognize negative samples in the U data and then feed them to classical positive-negative (PN) models. However, these approaches depend heavily on the heuristic strategies and often yield a poor solution. The other type, including \cite{liu2003building,lee2003learning}, offers a better solution by treating U data to be N data with a decayed weight. Nevertheless, finding an optimal weight turns out to be quite costly. Most importantly, the classifiers trained based on above approaches suffer from a systematic estimation bias \cite{du2015convex,kiryo2017positive}.

Seeking for unbiased PU classifier, \cite{du2014analysis} investigated the strategy of viewing U data as a weighted mixture of P and N data \cite{elkan2008learning}, and introduced an unbiased risk estimator by exploiting some non-convex symmetric losses, i.e., the ramp loss. Although cancelling the bias, the non-convex loss is undesirable for PU due to the difficulty of non-convex optimization. To this end, \cite{du2015convex} proposed a more general risk estimator which is always unbiased and convex if the convex loss satisfies a linear-odd condition \cite{patrini2016loss}. Theoretically, they argues the estimator yields globally optimal solution, with more appealing learning properties than the non-convex counterpart. More recently, \cite{kiryo2017positive} observed that the aforementioned unbiased risk estimators can go negative without bounding from the below, leading to serious overfitting when the classifier becomes too flexible. To fix this, they presented a non-negative biased risk estimator yet with favorable theoretical guarantees in terms of consistency, mean-squared-error reduction and estimation error. The proposed estimator is shown to be more robust against overfitting than previous unbiased ones. However, given limited P data, the overfitting issue still exists especially when very flexible deep neural network is applied.

Generative models, on the other hand, have the advantage in expressing complex data distribution. Apart from distribution density estimation, generative models are often applied to learn a function that is able to create more samples from the approximate distribution. Lately, a large body of successful deep generative models have emerged, especially generative adversarial networks (GAN) \cite{goodfellow2014generative,salimans2016improved}. GAN intends to solve the task of generative modeling by making two agents play a game against each other. One agent named generator synthesizes fake data from random noise; the other agent, termed as discriminator, examines both real and fake data and determines whether it is real or not. Both agents keep evolving over time and get better and better at their jobs. Eventually, the generator is forced to create synthetic data which is as realistic as possible to those from the training dataset.

Inspired by the tremendous success and expressive power of GAN, we novelly attack the binary PU classification task by resorting to generative modeling, and propose our generative positive-unlabeled (GenPU) learning framework. Building upon GAN, our GenPU model includes an array of generators and discriminators as agents in the game. These agents are devised to play different parts in simultaneously generating positive and negative real-like samples, and thereafter a standard PN classifier can be trained on those synthetic samples. Given a small portion of labeled P data as seeds, GenPU is able to capture the underlying P and N data distributions, with the capability to create infinite diverse P and N samples streams. In this way, the overfitting problem of conventional PU can be greatly mitigated. Furthermore, our GenPU is generalizable in the sense that it can be established by switching to different underlying GAN variants with distance metrics (i.e. Wasserstein GAN \cite{arjovsky2017wasserstein}) other than Jensen-Shannon divergence (JSD). As long as those variants are sophisticated to produce high-quality diverse samples, the optimal accuracy could be achieved by training a very deep neural networks.

Our main contribution$\colon$(i) we are the first to invent a totally new paradigm to effectively solve the PU task through deep generative models; (ii) we provide theoretical analysis to prove that, at equilibrium, our model is capable of learning both positive and negative data distributions; (iii) we experimentally show the effectiveness of the proposed model given limited P data on both synthetic and real-world dataset; (iiii) our method can be easily extended to solve the semi-supervised classification, and also opens a door to new solutions of many other weakly supervised learning tasks from the aspect of generative learning.

\section{Preliminaries}
\subsection{positive-unlabeled (PU) classification}
Given as input $d$-dimensional random variable $\textbf{x} \in \mathbb{R}^{d}$ and scalar random variable $y \in \{\pm1\}$ as class label, and let $p(\textbf{x}, y)$ be the \emph{joint density}, the \emph{class-conditional densities} are$\colon$
\begin{equation*}
  p_{p}(\textbf{x}) = p(\textbf{x}|y=1) \hspace{0.5cm} p_{n}(\textbf{x}) = p(\textbf{x}|y=-1),
\end{equation*}
while $p(\textbf{x})$ refers to as the unlabeled \emph{marginal density}.
The standard PU classification task \cite{ward2009presence} consists of a positive dataset $\mathcal{X}_{p}$ and an unlabeled dataset $\mathcal{X}_{u}$ with i.i.d samples drawn from $p_{p}(\textbf{x})$ and $p(\textbf{x})$, respectively$\colon$
\begin{equation*}
  \mathcal{X}_{p} = \{ \textbf{x}_{p}^{i} \}_{i=1}^{n_{p}} \sim p_{p}(\textbf{x}) \hspace{0.5cm} \mathcal{X}_{u} = \{ \textbf{x}_{u}^{i} \}_{i=1}^{n_{u}} \sim p(\textbf{x}).
\end{equation*}

Due to the fact that the unlabeled data can be regarded as a mixture of both positive and negative samples, the marginal density turns out to be
\begin{equation*}
  p(\textbf{x}) = \pi_{p} p(\textbf{x}|y=1) + \pi_{n} p(\textbf{x}|y=-1),
\end{equation*}
where $\pi_{p}=p(y=1)$ and $\pi_{n}= 1-\pi_{p}$ are denoted as \emph{class-prior probability}, which is usually unknown in advance and can be estimated from the given data \cite{jain2016estimating}. The objective of PU task is to train a classifier on $\mathcal{X}_{p}$ and $\mathcal{X}_{u}$ so as to classify the new unseen pattern $\textbf{x}^{new}$.

In particular, the empirical unbiased risk estimators introduced in \cite{du2014analysis,du2015convex} have a common formulation as
\begin{equation} \label{eq1}
    \widehat{R}_{pu}(g) = \pi_{p} \widehat{R}_{p}^{+}(g) + \widehat{R}_{u}^{-}(g) - \pi_{p} \widehat{R}_{p}^{-}(g),
\end{equation}
where $\widehat{R}_{p}^{+}$, $\widehat{R}_{u}^{-}$ and $\widehat{R}_{p}^{-}$ are the empirical version of the risks$\colon$
\begin{align*}
\begin{split}
    R_{p}^{+}(g) &= \mathbb{E}_{\textbf{x}\sim p_{p}(\textbf{x})} \ell (g(\textbf{x}), +1) \\
    R_{u}^{-}(g) &= \mathbb{E}_{\textbf{x}\sim p(\textbf{x})} \ell (g(\textbf{x}), -1) \\
    R_{p}^{-}(g) &= \mathbb{E}_{\textbf{x}\sim p_{p}(\textbf{x})} \ell (g(\textbf{x}), -1)
\end{split}
\end{align*}
with  $g \in \mathbb{R}^{d} \rightarrow \mathbb{R}$ be the binary classifier and $\ell \in \mathbb{R}^{d} \times \{ \pm 1\} \rightarrow \mathbb{R}$ be the loss function.

In contrast to PU classification, positive-negative (PN) classification assumes all negative samples,
\begin{equation*}
   \mathcal{X}_{n} = \{ \textbf{x}_{n}^{i} \}_{i=1}^{n_{n}} \sim p_{n}(\textbf{x}),
\end{equation*}
are labeled, so that the classifier can be trained in an ordinary supervised learning fashion.

\subsection{generative adversarial networks (GAN)}
GAN, originated in \cite{goodfellow2014generative}, is one of the most recent successful generative models that is equipped with the power of producing distributional outputs. GAN obtains this capability through an adversarial competition between a generator $G$ and a discriminator $D$ that involves optimizing the following minimax objective function$\colon$
\begin{multline} \label{eq2}
  \min_{G} \max_{D} \mathcal{V}(G, D) = \min_{G} \max_{D} \,\, \mathbb{E}_{\textbf{x}\sim p_{x}(\textbf{x})} \log(D(\textbf{x})) \\ + \mathbb{E}_{\textbf{z}\sim p_{z}(\textbf{z})} \log(1-D(G(\textbf{z}))),
\end{multline}
where $p_{x}(\textbf{x})$ represents true data distribution; $p_{z}(\textbf{z})$ is typically a simple prior distribution (e.g., $\mathcal{N}(0, 1)$) for latent code $\textbf{z}$, while a generator distribution $p_{g}(\textbf{x})$ associated with $G$ is induced by the transformation $G(\textbf{z})\colon\textbf{z}\rightarrow \textbf{x}$.

To find the optimal solution, \cite{goodfellow2014generative} employed simultaneous stochastic gradient descent (SGD) for alternately updating $D$ and $G$. The authors argued that, given the optimal $D$, minimizing $G$ is equivalent to minimizing the distribution distance between $p_{x}(\textbf{x})$ and $p_{g}(\textbf{x})$. At convergence, GAN has $p_{x}(\textbf{x})=p_{g}(\textbf{x})$.

\section{Generative PU Classification}
\subsection{problem setting}
Throughout the paper, $\{p_{p}(\textbf{x}), p_{n}(\textbf{x}), p(\textbf{x})\}$ denote the positive data distribution, the negative data distribution and the entire data distribution, respectively. Following the standard PU setting, we assume the data distribution has the form of $p(\textbf{x}) = \pi_{p} p_{p}(\textbf{x}) + \pi_{n} p_{n}(\textbf{x})$. $\mathcal{X}_{p}$ represents the positively labeled dataset while $\mathcal{X}_{u}$ serves as the unlabeled dataset. $\{D_{p}, D_{u}, D_{n}\}$ are referred to as the positive, unlabeled and negative discriminators, while $\{G_{p}, G_{n}\}$ stand for positive and negative generators, targeting to produce real-like positive and negative samples. Correspondingly, $\{p_{gp}(\textbf{x}), p_{gn}(\textbf{x})\}$ describe the positive and negative distributions induced by the generator functions $G_{p}(\textbf{z})$ and $G_{n}(\textbf{z})$.

\begin{figure}[t]
\begin{center}
   \includegraphics[width=0.9\linewidth]{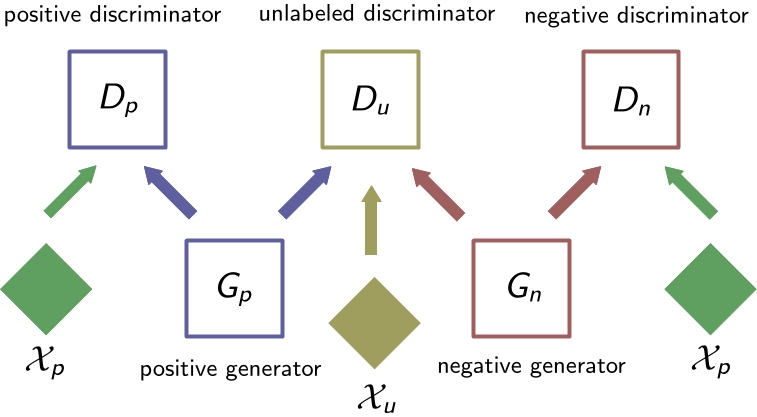}
\end{center}
\vspace{-0.2cm}
    \caption{\small Our GenPU framework. $D_{p}$ receives as inputs the real positive examples from $\mathcal{X}_{p}$ and the synthetic positive examples from $G_{p}$; $D_{n}$ receives as inputs the real positive examples from $\mathcal{X}_{p}$ and the synthetic negative examples from $G_{n}$; $D_{u}$ receives as inputs real unlabeled examples from $\mathcal{X}_{u}$, synthetic positive examples from $G_{p}$ as well as synthetic negative examples from $G_{n}$ at the same time. Associated with different loss functions, $G_{p}$ and $G_{n}$ are designated to generate positive and negative examples, respectively.}
\label{fig1}
\vspace{-0.2cm}
\end{figure}

\subsection{proposed GenPU model}
We build our GenPU model upon GAN by leveraging its massive potentiality in producing realistic data, with the goal of identification of both positive and negative distributions from P and U data. Then, a decision boundary can be made by training standard PN classifier on the generated samples. Fig.\ref{fig1} illustrates the architecture of the proposed framework.

In brief, GenPU framework is an analogy to a minimax game comprising of two generators $\{ G_{p}, G_{n}\}$ and three discriminators $\{D_{p}, D_{u}, D_{n}\}$. Guided by the adversarial supervision of $\{D_{p}, D_{u}, D_{n}\}$, $\{ G_{p}, G_{n}\}$ are tasked with synthesizing positive and negative samples that are indistinguishable with the real ones drawn from $\{p_{p}(\textbf{x}), p_{n}(\textbf{x})\}$, respectively. As being their competitive opponents, $\{D_{p}, D_{u}, D_{n}\}$ are devised to play distinct roles in instructing the learning process of $\{G_{p}, G_{n}\}$.

Among these discriminators, $D_{p}$ intends to discern the positive training samples from the fake positive samples outputted by $G_{p}$, whereas the business of $D_{u}$ is aimed at separating the unlabelled training samples from the fake samples of both $G_{p}$ and $G_{n}$. In the meantime, $D_{n}$ is conducted in a way that it can easily make a distinction between the positive training samples and the fake negative samples from $G_{n}$.

More formally, the overall GenPU objective function can be decomposed, in views of $G_{p}$ and $G_{n}$, as follows$\colon$
\begin{multline} \label{eq3}
    \min_{\{G_{p}, G_{n}\}} \max_{\{D_{p}, D_{u}, D_{n}\}} \mathcal{V}(G, D) = \pi_{p} \min_{G_{p}} \max_{\{D_{p}, D_{u}\}} \\ \mathcal{V}_{G_{p}}(G, D) + \pi_{n} \min_{G_{n}} \max_{\{D_{u}, D_{n}\}} \mathcal{V}_{G_{n}}(G, D),
\end{multline}
where $\pi_{p}$ and $\pi_{n}$ corresponding to $G_{p}$ and $G_{n}$ are the priors for positive class and negative class, satisfying $\pi_{p}+\pi_{n}=1$. Here, we assume $\pi_{p}$ and $\pi_{n}$ are predetermined and fixed.

The first term linked with $G_{p}$ in \eqref{eq3} can be further split into two standard GAN components $GAN_{G_{p}, D_{p}}$ and $GAN_{G_{p}, D_{u}}$:
\begin{multline} \label{eq4}
  \min_{G_{p}} \max_{\{D_{p}, D_{u}\}} \mathcal{V}_{G_{p}}(G, D) =
 \lambda_{p} \min_{G_{p}} \max_{D_{p}} \mathcal{V}_{G_{p}, D_{p}}(G, D) \\ +  \lambda_{u} \min_{G_{p}} \max_{D_{u}} \mathcal{V}_{G_{p}, D_{u}}(G, D),
\end{multline}
where $\lambda_{p}$ and $\lambda_{u}$ are the weights balancing the relative importance of effects between $D_{p}$ and $D_{u}$.
In particular, the value functions of $GAN_{G_{p}, D_{p}}$ and $GAN_{G_{p}, D_{u}}$ are
\begin{multline} \label{eq5}
   \mathcal{V}_{G_{p}, D_{p}}(G, D) = \mathbb{E}_{\textbf{x}\sim p_{p}(\textbf{x})} \log(D_{p}(\textbf{x}) \\ + \mathbb{E}_{\textbf{z}\sim p_{z}(\textbf{z})} \log(1-D_{p}(G_{p}(\textbf{z})))
\end{multline}
and
\begin{multline} \label{eq6}
   \mathcal{V}_{G_{p}, D_{u}}(G, D) = \mathbb{E}_{\textbf{x}\sim p_{u}(\textbf{x})} \log(D_{u}(\textbf{x})) \\ + \mathbb{E}_{\textbf{z}\sim p_{z}(\textbf{z})} \log(1-D_{u}(G_{p}(\textbf{z}))).
\end{multline}

On the other hand, the second term linked with $G_{n}$ in \eqref{eq3} can also be split into GAN components, namely $GAN_{G_{n}, D_{u}}$ and $GAN_{G_{n}, D_{n}}$:
\begin{multline} \label{eq7}
  \min_{G_{n}} \max_{\{D_{u}, D_{n}\}} \mathcal{V}_{G_{p}}(G, D) =
  \lambda_{u} \min_{G_{n}} \max_{D_{u}} \mathcal{V}_{G_{n}, D_{u}}(G, D) \\ + \lambda_{n} \min_{G_{n}} \max_{D_{n}} \mathcal{V}_{G_{n}, D_{n}}(G, D)
\end{multline}
whose weights $\lambda_{u}$ and $\lambda_{n}$ control the trade-off of $D_{u}$ and $D_{n}$. $GAN_{G_{n}, D_{u}}$ also takes the form of the standard GAN with the value function
\begin{multline} \label{eq8}
   \mathcal{V}_{G_{n}, D_{u}}(G, D) = \mathbb{E}_{\textbf{x}\sim p_{u}(\textbf{x})} \log(D_{u}(\textbf{x})) \\ + \mathbb{E}_{\textbf{z}\sim p_{z}(\textbf{z})} \log(1-D_{u}(G_{n}(\textbf{z})))
\end{multline}

In contrast to the ‘zero-sum’ loss applied elsewhere, the optimization of $GAN_{G_{n}, D_{n}}$ is given by first maximizing the objective
\begin{multline} \label{eq9}
   D_{n}^{\star} = \arg \max_{D_{n}} \mathbb{E}_{\textbf{x}\sim p_{p}(\textbf{x})} \log(D_{n}(\textbf{x})) \\ + \mathbb{E}_{\textbf{z}\sim p_{z}(\textbf{z})} \log(1-D_{n}(G_{n}(\textbf{z}))),
\end{multline}
then plugging $D_{n}^{\star}$ into the following equation $\eqref{eq10}$, leading to the value function w.r.t $G_{n}$ as
\begin{multline} \label{eq10}
    \mathcal{V}_{G_{n}, D_{n}^{\star}}(G, D_{n}^{\star}) = - \mathbb{E}_{\textbf{x}\sim p_{p}(\textbf{x})} \log(D_{n}^{\star}(\textbf{x})) \\ - \mathbb{E}_{\textbf{z}\sim p_{z}(\textbf{z})} \log(1-D_{n}^{\star}(G_{n}(\textbf{z}))).
\end{multline}

Intuitively, equations \eqref{eq5}-\eqref{eq6} indicate $G_{p}$, co-supervised under $D_{p}$ and $D_{u}$, endeavours to minimize the distance between the induced distribution $p_{gp}(\textbf{x})$ and positive data distribution $p_{p}(\textbf{x})$, while striving to stay around within the whole data distribution $p(\textbf{x})$. In fact, $G_{p}$ tries to deceive both discriminators by simultaneously maximizing $D_{p}$'s and $D_{u}$'s outputs on fake positive samples. As a result, the loss terms in \eqref{eq5} and \eqref{eq6} jointly guide $p_{gp}(\textbf{x})$ gradually moves towards and finally settles to $p_{p}(\textbf{x})$ of $p(\textbf{x})$.

On the other hand, equations \eqref{eq8}-\eqref{eq10}, suggest $G_{n}$, when facing both $D_{u}$ and $D_{n}$, struggles to make the induced $p_{gn}(\textbf{x})$ stay away from $p_{p}(\textbf{x})$, and also makes its effort to force $p_{gn}(\textbf{x})$ lie within $p(\textbf{x})$.

To achieve that, the objective in equation \eqref{eq10} favors $G_{n}$ to produce negative examples; this in turn helps $D_{n}$ to maximize the objective in \eqref{eq9} to separate positive training samples from fake negative samples rather than confusing $D_{n}$. Notice that, in the value function \eqref{eq10}, $G_{n}$ is designed to minimize $D_{n}$'s output instead of maximizing it when feeding $D_{n}$ with fake negative samples. Consequently, $D_{n}$ will send uniformly negative feedback to $G_{n}$. In this way, the gradient information derived from negative feedback drives down $p_{gn}(\textbf{x})$ near the positive data region $p_{p}(\textbf{x})$. In the meantime, the gradient signals from $D_{u}$ increase $p_{gn}(\textbf{x})$ outside the positive region but still restricting $p_{gn}(\textbf{x})$ in the true data distribution $p(\textbf{x})$. This crucial effect will eventually push $p_{gn}(\textbf{x})$ away from $p_{p}(\textbf{x})$ but towards $p_{n}(\textbf{x})$.

In practice, both discriminators $\{D_{p}, D_{u}, D_{n}\}$ and generators $\{G_{p}, G_{n}\}$ are implemented using deep neural networks parameterized by $\{\theta_{D_{p}}, \theta_{D_{u}}, \theta_{D_{n}}\}$ and $\{\theta_{G_{p}}, \theta_{G_{n}}\}$. The learning procedure of alternately updating between $\{\theta_{D_{p}}, \theta_{D_{u}}, \theta_{D_{n}}\}$ and $\{\theta_{G_{p}}, \theta_{G_{n}}\}$ via SGD are summarized in Algorithm \ref{alg1}. In particular, the tradeoff weights $\lambda_{p}$, $\lambda_{u}$ and $\lambda_{n}$ are the hyperparameters, while the class-priors $\pi_{p}$ and $\pi_{n}$ are known and fixed in advance.

\begin{algorithm}
\small
    \caption{generative positive-unlabeled (GenPU) learning}  \label{alg1}
\begin{algorithmic}[1]
\small
    \STATE {\bfseries Input$\colon$}positive weight $\lambda_{p}$, unlabeled weight $\lambda_{u}$ and negative weight $\lambda_{n}$
    \STATE {\bfseries Output$\colon$}discriminator parameters $\{\theta_{D_{p}}, \theta_{D_{u}}, \theta_{D_{n}}\}$ and generator parameters $\{\theta_{G_{p}}, \theta_{G_{n}}\}$
    \FOR{number of training iterations}
        \STATE \# \textbf{update discriminator networks $\{D_{p},D_{u},D_{n}\}$} \#
        \STATE sample minibatch of noise examples $\{\textbf{z}^{i} \}_{i=1}^{m}$ from noise prior $p_{z}(\textbf{z})$
        \STATE sample minibatch of positive examples $\{\textbf{x}_{p}^{i} \}_{i=1}^{m}$ from positive data distribution $p_{p}(\textbf{x})$
        \STATE sample minibatch of unlabeled examples $\{\textbf{x}_{u}^{i} \}_{i=1}^{m}$ from unlabeled data distribution $p_{u}(\textbf{x})$
        \STATE update the positive discriminator $D_{p}$ by ascending its stochastic gradient$\colon$ \\
        $\nabla_{\theta_{D_{p}}} \frac{1}{m} \sum_{i=1}^{m} \pi_{p} \lambda_{p} [\log(D_{p}(\textbf{x}_{p}^{i})) + \log(1-D_{p}(G_{p}(\textbf{z}^{i})))]$
        \STATE update the negative discriminator $D_{n}$ by ascending its stochastic gradient$\colon$ \\
        $\nabla_{\theta_{D_{n}}} \frac{1}{m} \sum_{i=1}^{m} \pi_{n} \lambda_{n} [\log(D_{n}(\textbf{x}_{p}^{i})) + \log(1-D_{n}(G_{n}(\textbf{z}^{i})))]$
        \STATE update the unlabeled discriminator $D_{u}$ by ascending its stochastic gradient$\colon$ \\
        $\nabla_{\theta_{D_{u}}} \frac{1}{m} \sum_{i=1}^{m} \lambda_{u} [\log(D_{u}(\textbf{x}_{u}^{i})) +$ \\
        $ \hspace{1.0cm} \pi_{p} \log(1-D_{u}(G_{p}(\textbf{z}^{i}))) + \pi_{n} \log(1-D_{u}(G_{n}(\textbf{z}^{i})))]$
        \STATE \# \textbf{update generator networks $\{G_{p},G_{n}\}$} \#
        \STATE sample minibatch of noise examples $\{\textbf{z}^{i} \}_{i=1}^{m}$ from noise prior $\mathbb{P}(\textbf{z})$
        \STATE update the positive generator $G_{p}$ by descending its stochastic gradient$\colon$ \\
        $\nabla_{\theta_{G_{p}}} \frac{1}{m} \sum_{i=1}^{m} \pi_{p} [- \lambda_{p} \log(D_{p}(G_{p}(\textbf{z}^{i})))$ \\
        $ \hspace{4.5cm} - \lambda_{u} \log(D_{u}(G_{p}(\textbf{z}^{i})))]$ \\
        \STATE update the negative generator $G_{n}$ by descending its stochastic gradient$\colon$ \\
        $\nabla_{\theta_{G_{n}}} \frac{1}{m} \sum_{i=1}^{m} \pi_{n} [- \lambda_{u} \log(D_{u}(G_{n}(\textbf{z}^{i}))) $ \\
        $ \hspace{4.5cm} \lambda_{n} \log(D_{n}(G_{n}(\textbf{z})))]$ \\
    \ENDFOR
\end{algorithmic}
\end{algorithm}

\subsection{theoretical analysis}
Theoretically, suppose all the $\{G_{p}, G_{n}\}$ and $\{D_{p}, D_{u}, D_{n}\}$ have enough capacity. Then the following results show that, at Nash equilibrium point of \eqref{eq3}, the minimal JSDs between the distributions induced by $\{G_{p}, G_{n}\}$ and data distributions $\{p_{p}(\textbf{x}), p_{n}(\textbf{x})\}$ are achieved, respectively, i.e., $p_{gp}(\textbf{x})=p_{p}(\textbf{x})$ and $p_{gn}(\textbf{x})=p_{n}(\textbf{x})$. Meanwhile, the JSD between the distribution induced by $G_{n}$ and data distribution $p_{p}(\textbf{x})$ is maximized, i.e., $p_{gn}(\textbf{x})$ almost never overlaps with $p_{p}(\textbf{x})$.
\begin{proposition}
Given fixed generators $G_{p}$, $G_{n}$ and known class prior $\pi_{p}$, the optimal discriminators $D_{p}$, $D_{u}$ and $D_{n}$ for the objective in equation \eqref{eq3} have the following forms$\colon$
\begin{equation*}
  D_{p}^{\star}(\textbf{x}) =  \frac{p_{p}(\textbf{x})}{p_{p}(\textbf{x}) + p_{gp}(\textbf{x})},
\end{equation*}
\begin{equation*}
  D_{u}^{\star}(\textbf{x}) =  \frac{p(\textbf{x})}{p(\textbf{x}) + \pi_{p} p_{gp}(\textbf{x}) + \pi_{n} p_{gn}(\textbf{x})}
\end{equation*}
and
\begin{equation*}
  D_{n}^{\star}(\textbf{x}) =  \frac{p_{p}(\textbf{x})}{p_{p}(\textbf{x}) + p_{gn}(\textbf{x})}.
\end{equation*}
\end{proposition}
\begin{proof}
Assume that all the discriminators $D_{p}$, $D_{u}$ and $D_{n}$ can be optimized in functional space. Differentiating the objective $\mathcal{V}(G, D)$ in \eqref{eq3} w.r.t. $D_{p}$, $D_{u}$ and $D_{n}$ and equating the functional derivatives to zero, we can obtain the optimal $D_{p}^{\star}$, $D_{u}^{\star}$ and $D_{n}^{\star}$ as described above.
\end{proof}
\begin{theorem}
Suppose the data distribution $p(\textbf{x})$ in the standard PU learning setting takes form of $p(\textbf{x}) = \pi_{p} p_{p}(\textbf{x}) + \pi_{n} p_{n}(\textbf{x})$, where $p_{p}(\textbf{x})$ and $p_{n}(\textbf{x})$ are well-separated. Given the optimal $D_{p}^{\star}$, $D_{u}^{\star}$ and $D_{n}^{\star}$, the minimax optimization problem in \eqref{eq3} obtains its optimal solution if
\begin{equation}  \label{eq11}
  p_{gp}(\textbf{x})=p_{p}(\textbf{x}) \hspace{0.2cm} \text{and} \hspace{0.2cm} p_{gn}(\textbf{x})=p_{n}(\textbf{x}),
\end{equation}
with the objective value of $-(\pi_{p} \lambda_{p} + \lambda_{u}) \log(4)$.
\end{theorem}
\begin{proof}
Substituting the optimal $D_{p}^{\star}$, $D_{u}^{\star}$ and $D_{n}^{\star}$ into \eqref{eq3}, the objective can be rewritten as follows$\colon$
\begin{multline} \label{eq12}
  \mathcal{V}(G, D^{\star}) = \pi_{p} \cdot \{ \lambda_{p} \cdot [ \mathbb{E}_{\textbf{x}\sim p_{p}(\textbf{x})} \log(\frac{p_{p}(\textbf{x})}{p_{p}(\textbf{x}) + p_{gp}(\textbf{x})}) \\ + \mathbb{E}_{\textbf{x}\sim p_{gp}(\textbf{x})} \log(\frac{p_{gp}(\textbf{x})}{p_{p}(\textbf{x}) + p_{gp}(\textbf{x})}) ]
  \\ + \lambda_{u} \cdot  [
  \mathbb{E}_{\textbf{x}\sim p_{u}(\textbf{x})} \log(\frac{p(\textbf{x})}{p(\textbf{x}) + \pi_{p} p_{gp}(\textbf{x}) + \pi_{n} p_{gn}(\textbf{x})}) \\ + \mathbb{E}_{\textbf{x}\sim p_{gp}(\textbf{x})} \log(\frac{\pi_{p} p_{gp}(\textbf{x}) + \pi_{n} p_{gn}(\textbf{x})}{p(\textbf{x}) + \pi_{p} p_{gp}(\textbf{x}) + \pi_{n} p_{gn}(\textbf{x})})] \} \\ + \pi_{n} \cdot \{ \lambda_{u} \cdot [
  \mathbb{E}_{\textbf{x}\sim p_{u}(\textbf{x})} \log(\frac{p(\textbf{x})}{p(\textbf{x}) + \pi_{p}  p_{gp}(\textbf{x}) + \pi_{n} p_{gn}(\textbf{x})}) \\ + \mathbb{E}_{\textbf{x}\sim p_{gn}(\textbf{x})} \log(\frac{\pi_{p} p_{gp}(\textbf{x}) + \pi_{n} p_{gn}(\textbf{x})}{p(\textbf{x}) + \pi_{p} p_{gp}(\textbf{x}) + \pi_{n} p_{gn}(\textbf{x})})]
  \\ - \lambda_{n} \cdot [\mathbb{E}_{\textbf{x}\sim p_{p}(\textbf{x})} \log(\frac{p_{p}(\textbf{x})}{p_{p}(\textbf{x}) + p_{gn}(\textbf{x})}) \\ + \mathbb{E}_{\textbf{x}\sim p_{gn}(\textbf{x})} \log(\frac{p_{gn}(\textbf{x})}{p_{p}(\textbf{x}) + p_{gn}(\textbf{x})})] \}
\end{multline}

Combining the intermediate terms associated with $\lambda_{u}$ using the fact $\pi_{p} + \pi_{n} = 1$, we reorganize \eqref{eq12} and arrive at
\begin{multline} \label{eq13}
  G^{\star} = \arg\min_{G} \mathcal{V}(G, D^{\star})
   \\ = \arg\min_{G} \,\, \pi_{p} \cdot \lambda_{p} \cdot  [2 \cdot \text{JSD}(p_{p} \, \Vert \, p_{gp}) - \log(4)]
   \\ + \lambda_{u} \cdot [2 \cdot \text{JSD}(p \, \Vert \, \pi_{p} p_{gp} + \pi_{n} p_{gn}) - \log(4)] \\ - \pi_{n} \cdot \lambda_{n} \cdot [2 \cdot \text{JSD}(p_{p} \, \Vert \, p_{gn}) - \log(4)],
\end{multline}
which peaks its minimum if
\begin{equation} \label{eq14}
    p_{gp}(\textbf{x})=p_{p}(\textbf{x}),
\end{equation}
\begin{equation} \label{eq15}
    \pi_{p} \cdot  p_{gp}(\textbf{x}) + \pi_{n} \cdot p_{gn}(\textbf{x}) = p(\textbf{x})
\end{equation}
and for almost every \textbf{x} except for those in a zero measure set
\begin{equation} \label{eq16}
  p_{p}(\textbf{x}) > 0 \Rightarrow p_{gn}(\textbf{x}) = 0, \,\,
  p_{gn}(\textbf{x}) > 0 \Rightarrow p_{p}(\textbf{x}) = 0.
\end{equation}
The solution to $G = \{G_{p}, G_{n}\}$ must jointly satisfy the conditions described in \eqref{eq14}, \eqref{eq15} and \eqref{eq16}, which implies \eqref{eq11} and leads to the minimum objective value of $-(\pi_{p} \lambda_{p} + \lambda_{u}) \log(4)$.
\end{proof}

The theorem reveals that approaching to Nash equilibrium is equivalent to jointly minimizing $\text{JSD}(p \, \Vert \, \pi_{p} p_{gp} + \pi_{n} p_{gn})$ and $\text{JSD}(p_{p} \, \Vert \, p_{gp})$ and maximizing $\text{JSD}(p_{p} \, \Vert \, p_{gn})$ at the same time, thus exactly capturing $p_{p}$ and $p_{n}$.

\subsection{extension to semi-supervised classification}
The goal of semi-supervised classification is to learn a classifier from positive, negative and unlabeled data. In such context, besides training sets $\mathcal{X}_{p}$ and $\mathcal{X}_{u}$, a partially labeled negative set $\mathcal{X}_{n}$ is also available, with samples drawn from negative data distribution $p_{n}(\textbf{x})$.

In fact, the very same architecture of GenPU can be applied to the semi-supervised classification task by just adapting the standard GAN value function to $G_{n}$, then the total value function turns out to be
\begin{multline} \label{eq17}
  \mathcal{V}(G, D) =
  \\ \pi_{p} \cdot \{ \lambda_{p} \cdot [ \mathbb{E}_{\textbf{x}\sim p_{p}(\textbf{x})} \log(D_{p}(\textbf{x})) + \mathbb{E}_{\textbf{z}\sim p_{z}(\textbf{z})} \log(1-D_{p}(G_{p}(\textbf{z}))) ]
  \\ + \lambda_{u} \cdot  [
  \mathbb{E}_{\textbf{x}\sim p_{u}(\textbf{x})} \log(D_{u}(\textbf{x})) + \mathbb{E}_{\textbf{z}\sim p_{z}(\textbf{z})} \log(1-D_{u}(G_{p}(\textbf{z})))] \}
  \\ +
  \\ \pi_{n} \cdot \{ \lambda_{u} \cdot [
  \mathbb{E}_{\textbf{x}\sim p_{u}(\textbf{x})} \log(D_{u}(\textbf{x})) + \mathbb{E}_{\textbf{z}\sim p_{z}(\textbf{z})} \log(1-D_{u}(G_{n}(\textbf{z})))]
  \\ + \lambda_{n} \cdot [ \mathbb{E}_{\textbf{x}\sim p_{n}(\textbf{x})} \log(D_{n}(\textbf{x})) + \mathbb{E}_{\textbf{z}\sim p_{z}(\textbf{z})} \log(1-D_{n}(G_{n}(\textbf{z}))) ] \},
\end{multline}
Under formulation \eqref{eq17}, $D_{n}$ discriminates the negative training samples from the synthetic negative samples produced by $G_{n}$. Now $G_{n}$ intends to fool $D_{n}$ and $D_{u}$ simultaneously by outputting realistic examples, just like $G_{p}$ does for $D_{p}$ and $D_{u}$. Being attracted by both $p(\textbf{x})$ and $p_{n}(\textbf{x})$ , the induced distribution $p_{gn}(\textbf{x})$ slightly approaches to $p_{n}(\textbf{x})$ and finally recovers the true distribution $p_{n}(\textbf{x})$ of $p(\textbf{x})$. Theoretically, it is not hard to show the optimal $G_{p}$ and $G_{n}$, at convergence, give rise to $p_{gp}(\textbf{x})=p_{p}(\textbf{x})$ and $p_{gn}(\textbf{x})=p_{n}(\textbf{x})$.

\section{Related Work}
A variety of approaches involving training ensembles of GANs have been recently explored. Among them, D2GAN \cite{nguyen2017dual} adopted two discriminators to jointly minimize Kullback-Leibler (KL) and reserve KL divergences, so as to exploit the complementary statistical properties of the two divergences to diversify the data generation. For the purpose of stabilizing GAN training, \cite{neyshabur2017stabilizing} proposed to train one generator against a set of discriminators. Each discriminator pays attention to a different random projection of the data, thus avoiding perfect detection of the fake samples. By doing so, the issue of vanishing gradients for discriminator could be effectively alleviated since the low-dimensional projections of the true data distribution are less concentrated. A similar $N$-discriminator extension to GAN was employed in \cite{durugkar2016generative} with the objective to accelerate training of generator to a more stable state with higher quality output, by using various strategies of aggregating feedbacks from multiple discriminators. Inspired by boosting techniques, \cite{wang2016ensembles} introduced an heuristic strategy based on the cascade of GANs to mitigate the missing modes problem. Addressing the same issue, \cite{tolstikhin2017adagan} developed an additive procedure to incrementally train a mixture of generators, each of which is assumed to cover some modes well in data space. \cite{ghosh2017multi} presented MAD-GAN, an architecture that incorporates multiple generators and diversity enforcing terms, which has been shown to capture diverse modes with plausible samples. Rather than using diversity enforcing terms, MGAN \cite{hoang2017multi} employed an additional multi-class classifier that allows it to identify which source the synthetic sample comes from. This modification leads to the effect of maximizing JSD among generator distributions, thus encouraging generators to specialize in different data modes.

The common theme in aforementioned variants mainly focus on improving the training properties of GAN, e.g., voiding mode collapsing problem. By contrast, our framework fundamentally differs from them in terms of motivation, application and architecture. Our model is specialized in PU classification task and can potentially be extended to other tasks in the context of weakly supervised learning. Nevertheless, we do share a point in common in sense that multiple discriminators or generators are exploited, which is shown to be more stable and easier to train than the original GAN.

\section{Experimental Results}
We show the efficacy of our framework by conducting experiments on synthetic and real-world images datasets. For real data, the approaches including oracle PN, unbiased PU (UPU) \cite{du2015convex}, non-negative PU (NNPU) \cite{kiryo2017positive} \footnote{The software codes for UPU and NNPU are downloaded from https://github.com/kiryor/nnPUlearning} are selected for comparison. Specifically, the Oracle PN means all the training labels are available for all the P and N data, whose performance is just used as a reference for other approaches. As for the UPU and NNPU, the true class-prior $\pi_{p}$ is assumed to be known in advance. For practical implementation, we apply the non-saturating heuristic loss \cite{goodfellow2014generative} to the underlying standard GAN, when the multi-layer perceptron (MLP) network is used.

\begin{figure} [t]
  \begin{subfigure}[b]{0.19\linewidth}
    \includegraphics[trim=0.3cm 0.6cm 0.3cm 0.6cm, width=1.8cm, clip]{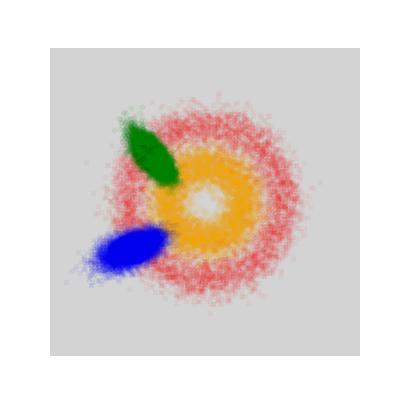}
  \end{subfigure}
  \begin{subfigure}[b]{0.19\linewidth}
    \includegraphics[trim=0.3cm 0.6cm 0.3cm 0.6cm, width=1.8cm, clip]{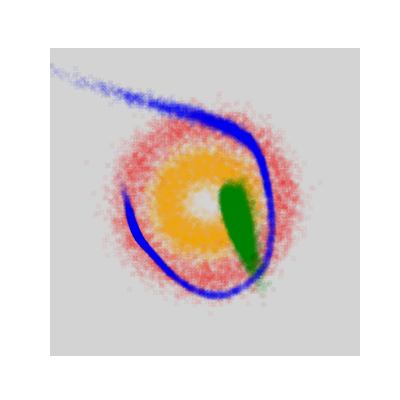}
  \end{subfigure}
  \begin{subfigure}[b]{0.19\linewidth}
    \includegraphics[trim=0.3cm 0.6cm 0.3cm 0.6cm, width=1.8cm, clip]{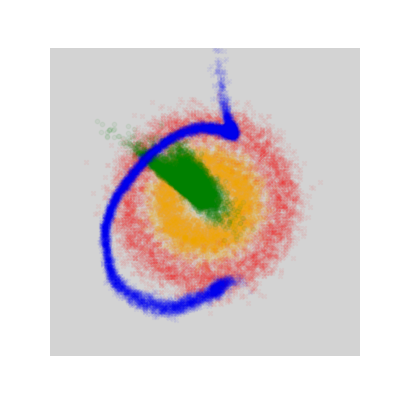}
  \end{subfigure}
  \begin{subfigure}[b]{0.19\linewidth}
    \includegraphics[trim=0.3cm 0.6cm 0.3cm 0.6cm, width=1.8cm, clip]{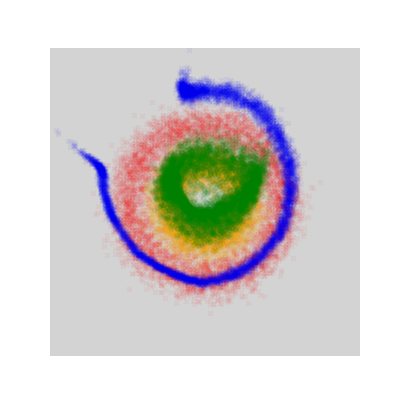}
  \end{subfigure}
  \begin{subfigure}[b]{0.19\linewidth}
    \includegraphics[trim=0.3cm 0.6cm 0.3cm 0.6cm, width=1.8cm, clip]{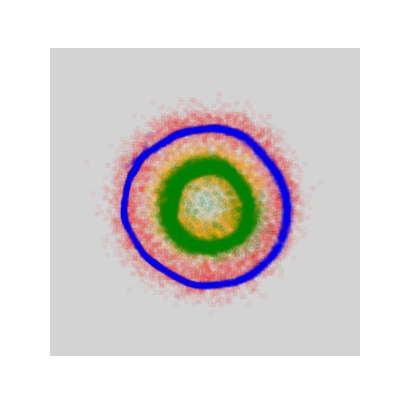}
  \end{subfigure}

  \begin{subfigure}[b]{0.19\linewidth}
    \includegraphics[trim=0.3cm 0.6cm 0.3cm 0.6cm, width=1.8cm, clip]{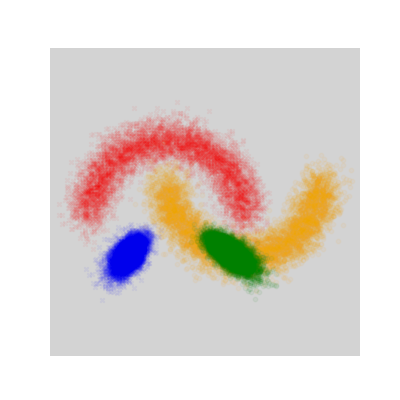}
  \end{subfigure}
  \begin{subfigure}[b]{0.19\linewidth}
    \includegraphics[trim=0.3cm 0.6cm 0.3cm 0.6cm, width=1.8cm, clip]{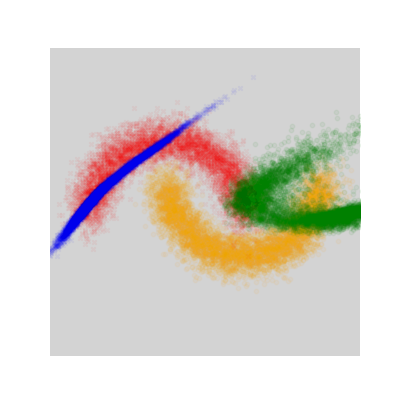}
  \end{subfigure}
  \begin{subfigure}[b]{0.19\linewidth}
    \includegraphics[trim=0.3cm 0.6cm 0.3cm 0.6cm, width=1.8cm, clip]{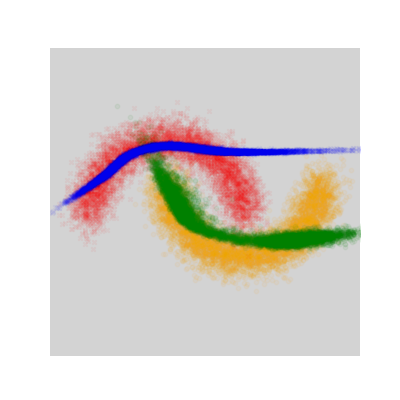}
  \end{subfigure}
  \begin{subfigure}[b]{0.19\linewidth}
    \includegraphics[trim=0.3cm 0.6cm 0.3cm 0.6cm, width=1.8cm, clip]{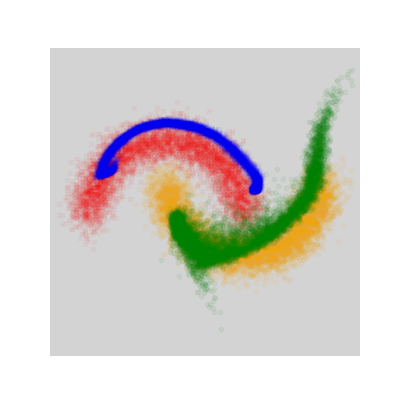}
  \end{subfigure}
  \begin{subfigure}[b]{0.19\linewidth}
    \includegraphics[trim=0.3cm 0.6cm 0.3cm 0.6cm, width=1.8cm, clip]{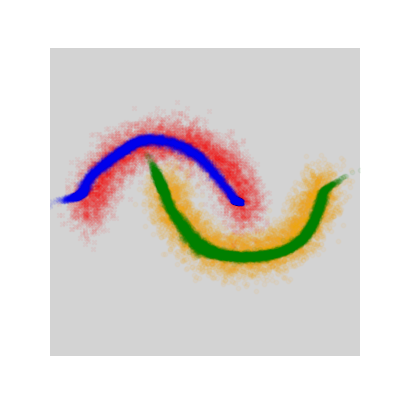}
  \end{subfigure}

  \begin{subfigure}[b]{0.19\linewidth}
    \includegraphics[trim=0.3cm 0.6cm 0.3cm 0.6cm, width=1.8cm, clip]{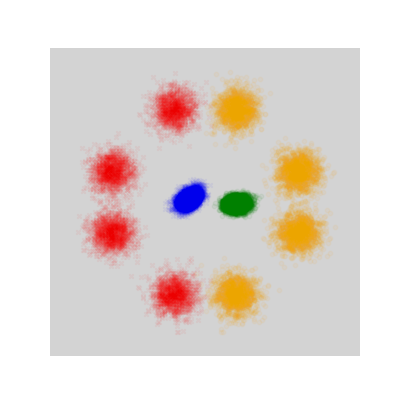}
  \end{subfigure}
  \begin{subfigure}[b]{0.19\linewidth}
    \includegraphics[trim=0.3cm 0.6cm 0.3cm 0.6cm, width=1.8cm, clip]{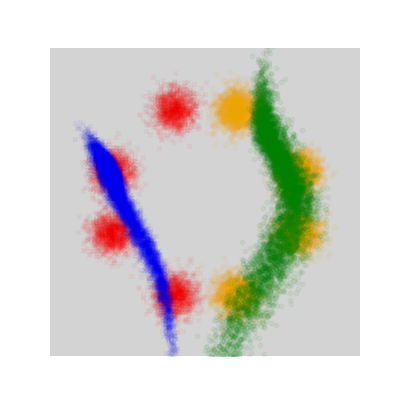}
  \end{subfigure}
  \begin{subfigure}[b]{0.19\linewidth}
    \includegraphics[trim=0.3cm 0.6cm 0.3cm 0.6cm, width=1.8cm, clip]{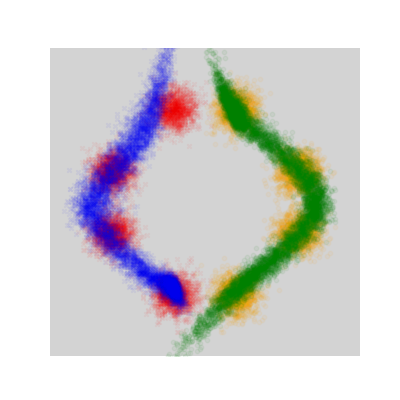}
  \end{subfigure}
  \begin{subfigure}[b]{0.19\linewidth}
    \includegraphics[trim=0.3cm 0.6cm 0.3cm 0.6cm, width=1.8cm, clip]{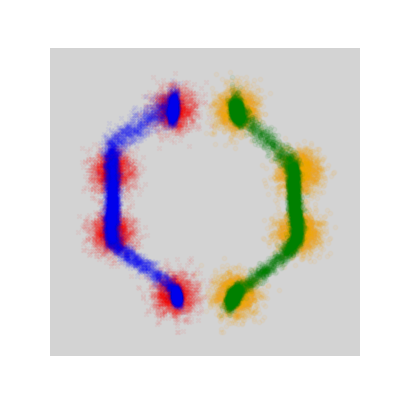}
  \end{subfigure}
  \begin{subfigure}[b]{0.19\linewidth}
    \includegraphics[trim=0.3cm 0.6cm 0.3cm 0.6cm, width=1.8cm, clip]{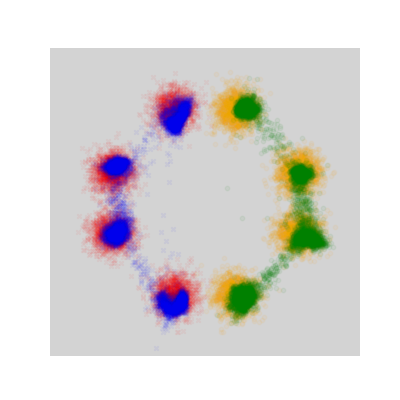}
  \end{subfigure}

  \begin{subfigure}[b]{0.19\linewidth}
    \includegraphics[trim=0.3cm 0.6cm 0.3cm 0.6cm, width=1.8cm, clip]{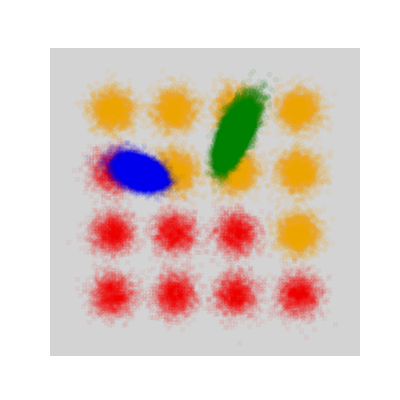}
  \end{subfigure}
  \begin{subfigure}[b]{0.19\linewidth}
    \includegraphics[trim=0.3cm 0.6cm 0.3cm 0.6cm, width=1.8cm, clip]{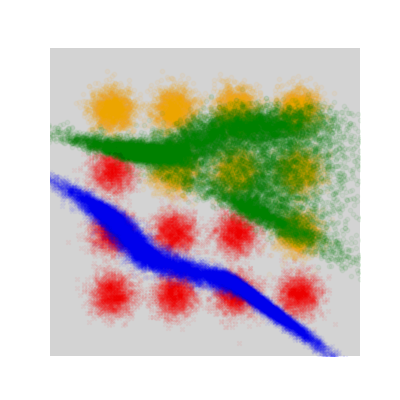}
  \end{subfigure}
  \begin{subfigure}[b]{0.19\linewidth}
    \includegraphics[trim=0.3cm 0.6cm 0.3cm 0.6cm, width=1.8cm, clip]{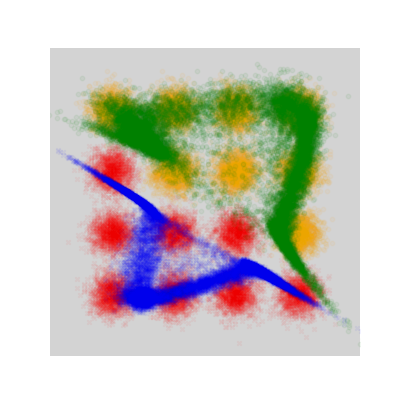}
  \end{subfigure}
  \begin{subfigure}[b]{0.19\linewidth}
    \includegraphics[trim=0.3cm 0.6cm 0.3cm 0.6cm, width=1.8cm, clip]{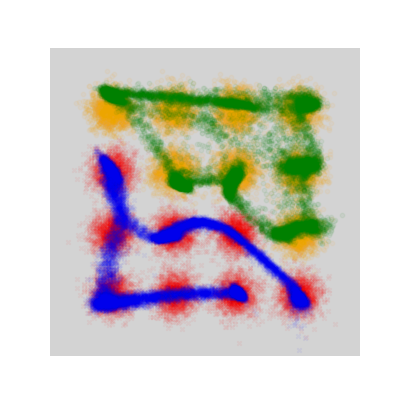}
  \end{subfigure}
  \begin{subfigure}[b]{0.19\linewidth}
    \includegraphics[trim=0.3cm 0.6cm 0.3cm 0.6cm, width=1.8cm, clip]{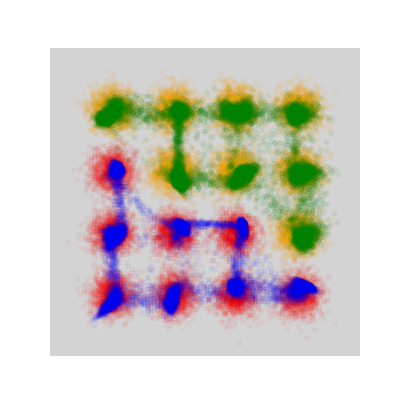}
  \end{subfigure}
\caption{\small Evolution of the positive samples (in green) and negative samples (in blue) produced by GenPU. The true positive samples (in orange) and true negative samples (in red) are also illustrated.} \label{fig2}
\vspace{-0.2cm}
\end{figure}
\subsection{synthetic simulation}
We begin our test with a toy example to visualize the learning behaviors of our GenPU. The training samples are synthesized using concentric circles, Gaussian mixtures and half moons functions with Gaussian noises added to the data (standard deviation is $0.1414$). The training set contains $5000$ positive and $5000$ negative samples, which are then partitioned into $500$ positively labelled and $9500$ unlabelled samples. We establish the generators with two fully connected hidden layers and the discriminators with one hidden layer. There are $128$ ReLU units contained in all hidden layers. The dimensionality of the input latent code is set to $256$. Fig.\ref{fig2} depicts the evolution of positive and negative samples produced by GenPU through time. As expected, in all the scenarios, the induced generator distributions successfully converge to the respective true data distributions given limited P data. Notice that the Gaussian mixtures cases demonstrate the capability of our GenPU to learn a distribution with multiple submodes.

\begin{table} [!ht]
\begin{center}
{ \scriptsize
    \begin{tabular}{ccc}
    \hline\noalign{\smallskip}
    Operation & Feature Maps & Nonlinearity \\
    \noalign{\smallskip}\hline\noalign{\smallskip}
    $G_{p}(\textbf{z}),G_{n}(\textbf{z})\colon \textbf{z} \sim \mathcal{N}(0, I) $ & 100 & \\
    fully connected & 256 & leaky relu \\
    fully connected & 256 & leaky relu \\
    fully connected & 256/784 & tanh \\ \hline
    $D_{p}(\textbf{x}), D_{n}(\textbf{x})$ & 256/784 &  \\
    fully connected & 1 & sigmoid \\ \hline
    $D_{u}(\textbf{x})$ & 256/784 &  \\
    fully connected & 256 & leaky relu \\
    fully connected & 256 & leaky relu \\
    fully connected & 1 & sigmoid \\ \hline
    leaky relu slope & 0.2 & \\
    mini-batch size for $\mathcal{X}_{p}$, $\mathcal{X}_{u}$ & 50, 100 & \\
    learning rate & 0.0003 & \\
    optimizer & Adam($0.9$, $0.999$) &   \\
    weight, bias initialization & 0, 0 & \\
    \noalign{\smallskip}\hline\noalign{\smallskip}
    \end{tabular}
}
\end{center}
\vspace{-0.2cm}
\caption{Specifications of network architecture and hyperparameters for USPS/MNIST dataset.} \label{tab1}
\vspace{-0.2cm}
\end{table}

\subsection{mnist and usps dataset}
Next, the evaluation is carried out on MNIST \cite{lecun1998the} and USPS \cite{lecun1990handwritten} datasets. For MNIST, we each time select a pair of digits to construct the P and N sets, each of which consists of $5,000$ training points. The specifics for architecture and hyperparameters are described in Tab.\ref{tab1}.
\begin{figure} [!b]
  \centering
  \begin{subfigure}[b]{0.495\linewidth}
    \includegraphics[trim=0.6cm 0.5cm 0.6cm 0.6cm, width=\linewidth, clip]{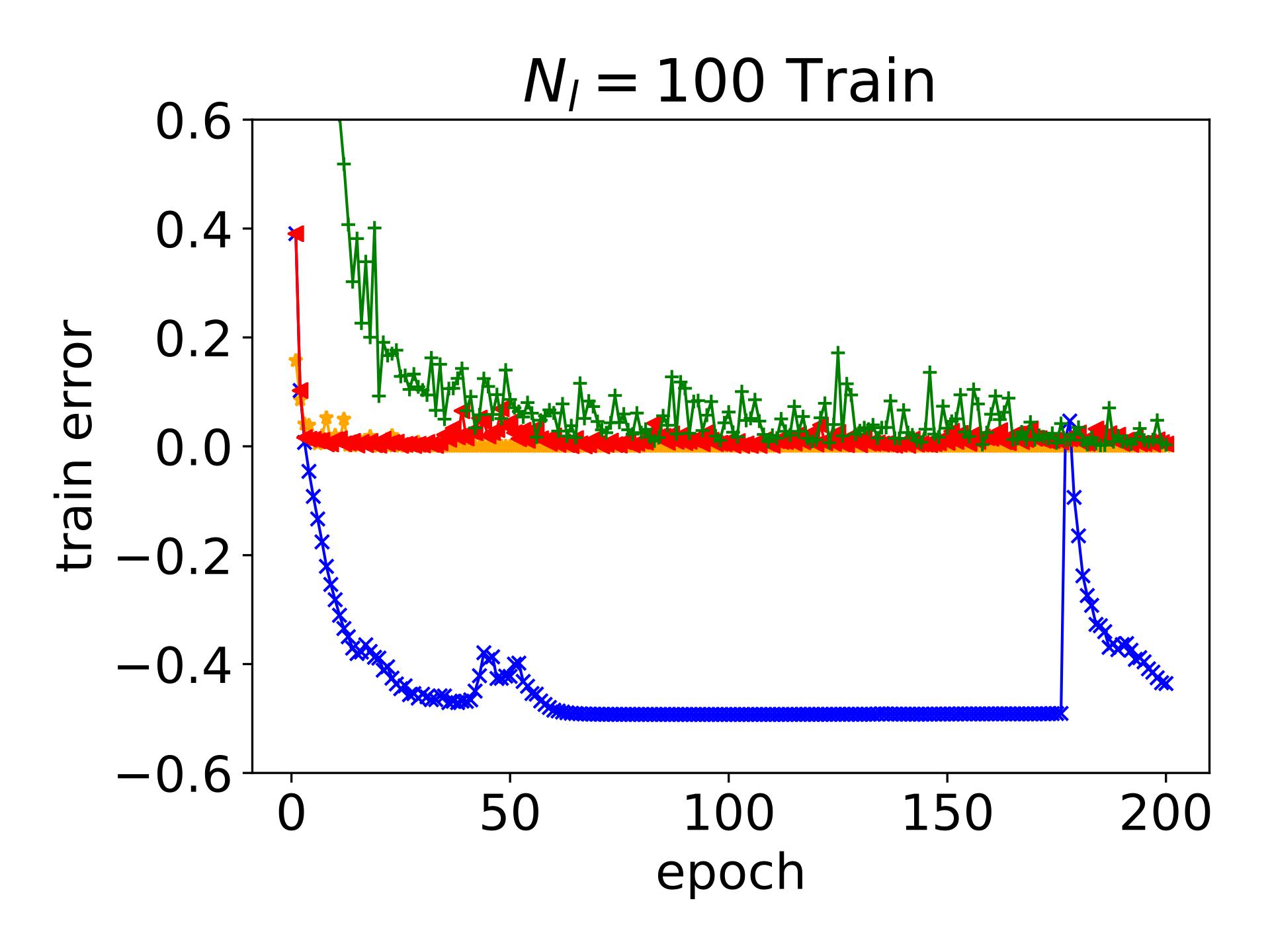}
  \end{subfigure}
  \begin{subfigure}[b]{0.495\linewidth}
    \includegraphics[trim=0.6cm 0.5cm 0.6cm 0.6cm, width=\linewidth, clip]{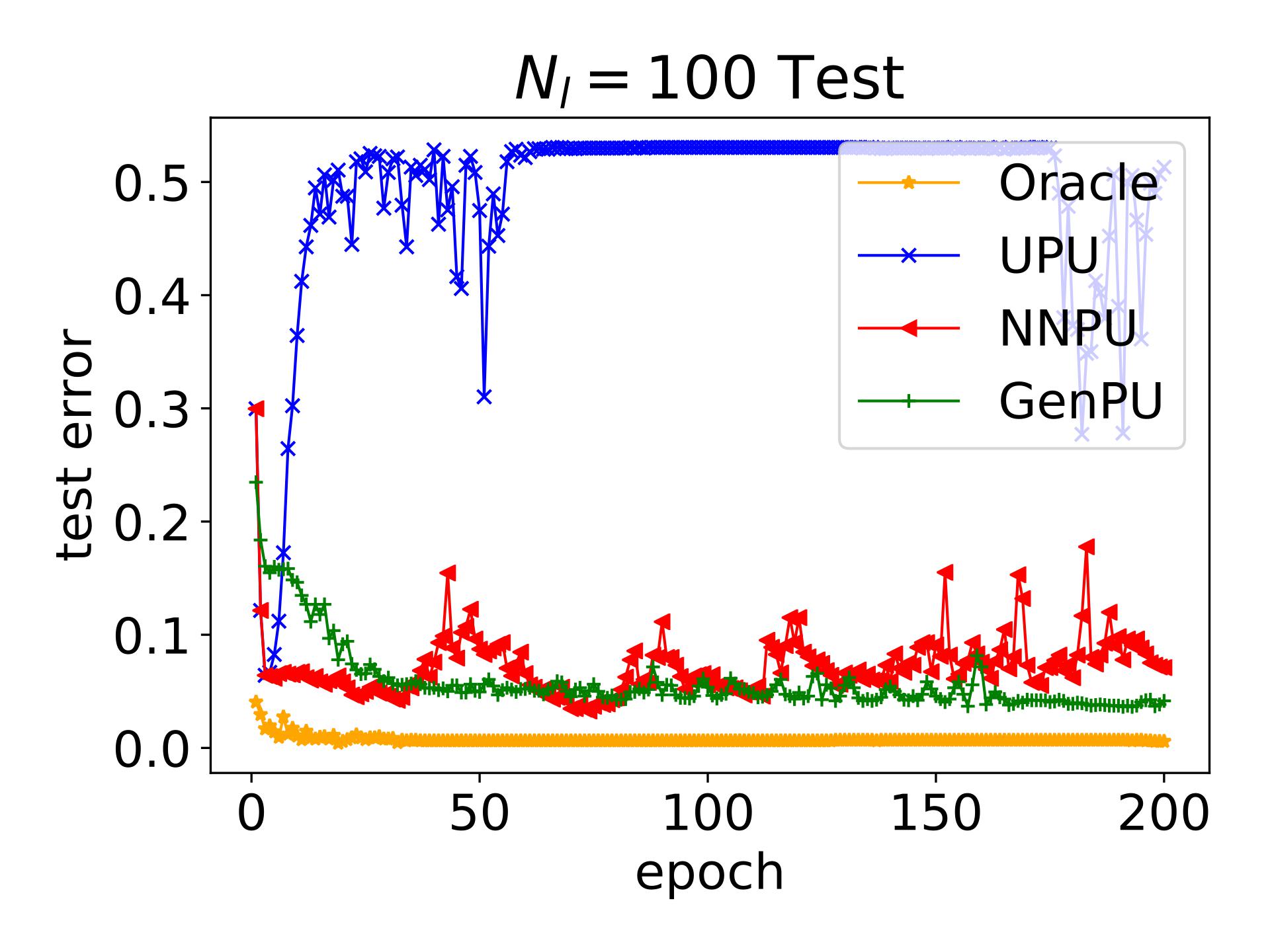}
  \end{subfigure}
  \begin{subfigure}[b]{0.495\linewidth}
    \includegraphics[trim=0.6cm 0.5cm 0.6cm 0.6cm, width=\linewidth, clip]{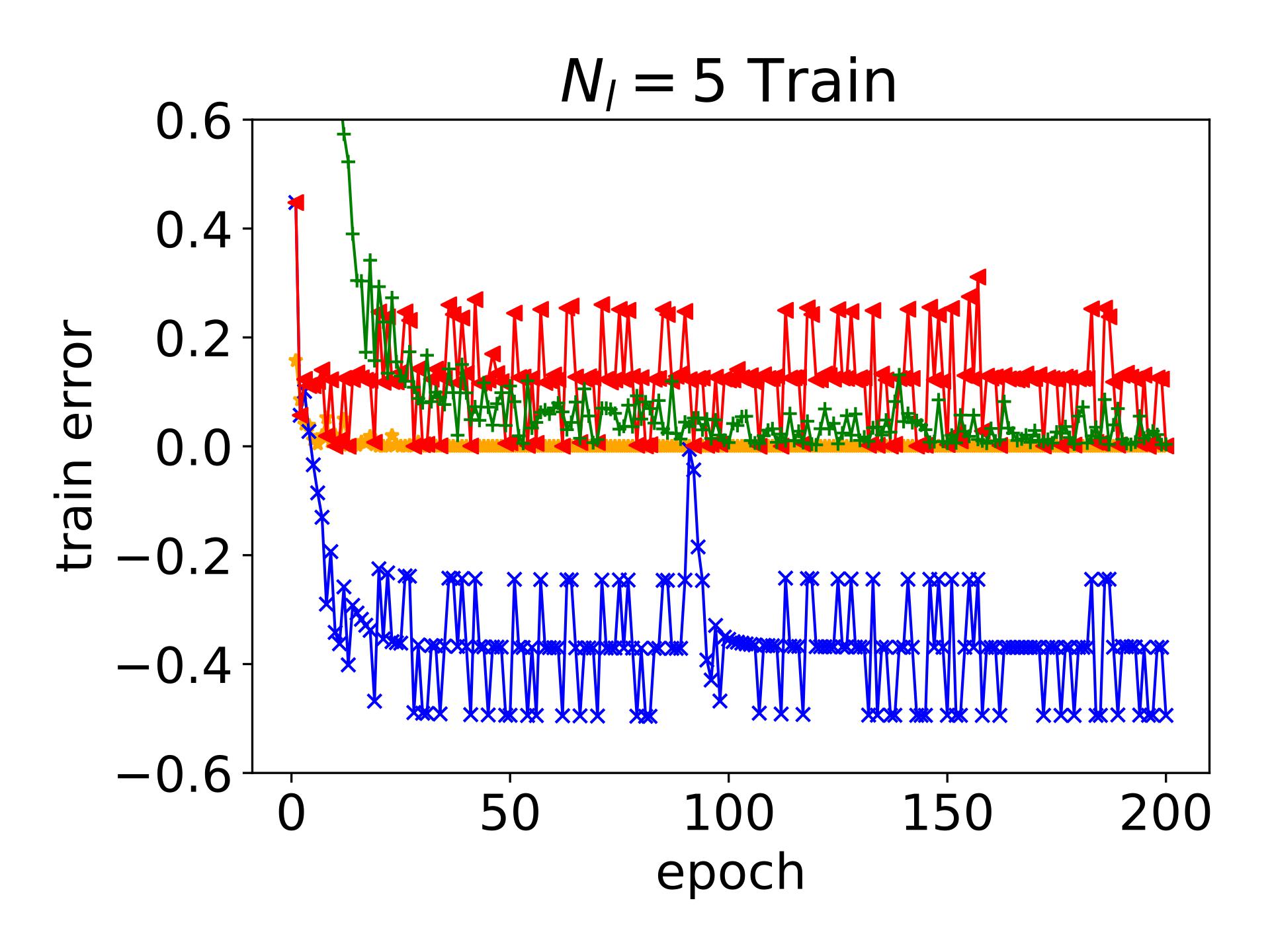}
  \end{subfigure}
  \begin{subfigure}[b]{0.495\linewidth}
    \includegraphics[trim=0.6cm 0.5cm 0.6cm 0.6cm, width=\linewidth, clip]{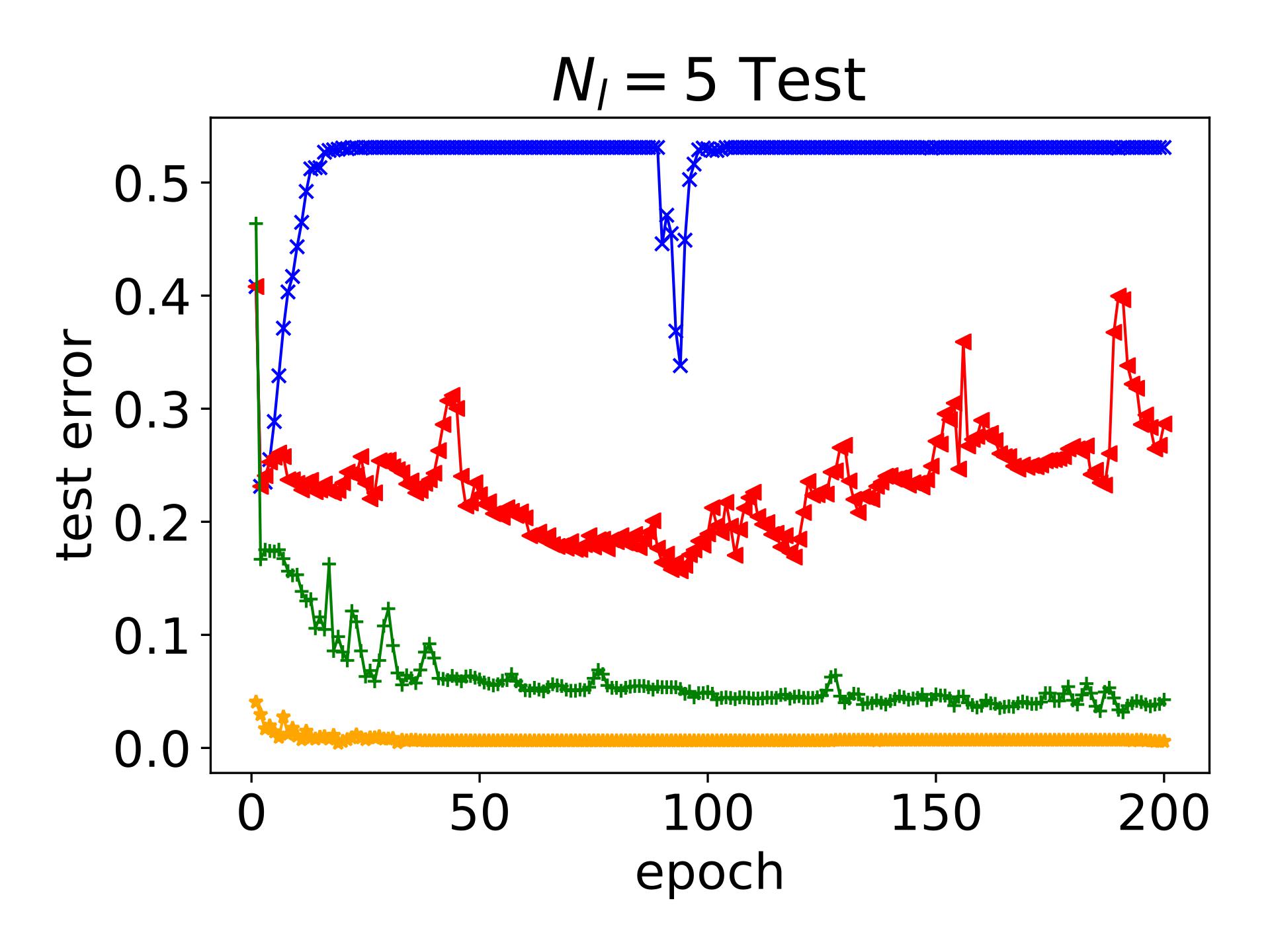}
  \end{subfigure}
\caption{Training error and test error of deep PN classifiers on MINST for the pair `$3$' vs `$5$' with distinct $N_{l}$. Top$\colon$(a) and (b) for $N_{l}=100$. Bottom$\colon$(c) and (d) for $N_{l}=5$.} \label{fig3}
\vspace{-0.2cm}
\end{figure}
To be challenging, the results of the most visually similar digit pairs, such as `$3$' vs `$5$' and `$8$' vs `$3$', are recorded in Tab.\ref{tab1}. The best accuracies are shown with the number of labeled positive examples $N_{l}$ ranging from $100$ to $1$. Obviously, our method outperforms UPU in all the cases. We also observe our GenPU achieves better than or comparable accuracy to NNPU when the number of labeled samples is relatively large (i.e., $N_{l}=100$). However, when the labeled samples are insufficient, for instance $N_{l}=5$ of the `$3$' vs `$5$' scenario, the accuracy of GenPU slightly decreases from $0.983$ to $0.979$, which is in contrast to that of NNPU drops significantly from $0.969$ to $0.843$. Spectacularly, GenPU still remains highly accurate even if only one labeled sample is provided whereas NNPU fails in this situation.

\begin{table*} [t]
\begin{center}
{\scriptsize
    \begin{tabular}{| c | c | c | c | c | c | c | c | c |}
    \hline
    \multicolumn{1}{|c|}{MNIST} & \multicolumn{4}{|c|}{`3' vs. `5'} & \multicolumn{4}{|c|}{`8' vs. `3'} \\ \hline
    $N_{p}\colon N_{u}$ & Oracle PN & UPU & NNPU & GenPU & Oracle PN & UPU & NNPU & GenPU \\ \hline
    \textbf{100}$\colon$9900 & 0.993 & 0.914 & 0.969 & \textbf{0.983} & 0.994 & 0.932 & 0.974 & \textbf{0.982} \\ \hline
    \textbf{50}$\colon$9950 & 0.993 & 0.854 & 0.966 & \textbf{0.982 }& 0.994 & 0.873 & 0.965 & \textbf{0.979} \\ \hline
    \textbf{10}$\colon$9990 & 0.993 & 0.711 & 0.866 & \textbf{0.980} & 0.994 & 0.733 & 0.907 & \textbf{0.978} \\ \hline
    \textbf{5}$\colon$9995 & 0.993 & 0.660 & 0.843 & \textbf{0.979} & 0.994 & 0.684 & 0.840 & \textbf{0.976} \\ \hline
    \textbf{1}$\colon$9999 & 0.993 & 0.557 & 0.563 & \textbf{0.976} & 0.994 & 0.550 & 0.573 & \textbf{0.972} \\ \hline
    \end{tabular}
}
\end{center}
\vspace{-0.1cm}
\caption{The accuracy comparison on MNIST for $N_{l}\in\{100, 50, 10, 5, 1\}$.} \label{tab2}
\vspace{-0.2cm}
\end{table*}

Fig.\ref{fig3} reports the training and test errors of the classifiers for distinct settings of $N_{l}$. When $N_{l}$ is $100$, UPU suffers from a serious overfitting to training data, whilst both NNPU and GenPU perform fairly well. As $N_{l}$ goes small (i.e., $5$), NNPU also starts to overfit. It should be mentioned that the negative training curve of UPU (in blue) is because the unbiased risk estimators  \cite{du2014analysis,du2015convex} in \eqref{eq1} contain negative loss term which is unbounded from the below. When the classifier becomes very flexible, the risk can be arbitrarily negative \cite{kiryo2017positive}. Additionally, the rather limited $N_{l}$ cause training processes of both UPU and NNPU behave unstable. In contrast, GenPU avoids overfitting to small training P data by restricting the models of $D_{p}$ and $D_{n}$ from being too complex when $N_{l}$ becomes small (see Tab.\ref{tab1}). For visualization, Fig.\ref{fig4} demonstrates the generated digits with only one labeled `3', together with the projected distributions induced by $G_{p}$ and $G_{n}$. In Tab.\ref{tab3}, similar results can be obtained on USPS data.

\begin{table} [t]
\begin{center}
{\scriptsize
    \begin{tabular}{| c | c | c | c | c | c | c |}
    \hline
    \multicolumn{1}{|c|}{USPS} & \multicolumn{3}{|c|}{`3' vs `5'} & \multicolumn{3}{|c|}{`8' vs `3'} \\ \hline
    $N_{l}\colon N_{u}$ & UPU & NNPU & GenPU & UPU & NNPU & GenPU \\ \hline
    \textbf{50}$\colon$1950 & 0.890 & \textbf{0.965} & \textbf{0.965} & 0.900 & \textbf{0.965} & 0.945 \\ \hline
    \textbf{10}$\colon$1990 & 0.735 & 0.880 & \textbf{0.955} & 0.725 & 0.920 & \textbf{0.935} \\ \hline
    \textbf{5}$\colon$1995 & 0.670 & 0.830 & \textbf{0.950} & 0.630 & 0.865 & \textbf{0.925} \\ \hline
    \textbf{1}$\colon$1999 & 0.540 & 0.610 & \textbf{0.940} & 0.555 & 0.635 & \textbf{0.920} \\ \hline
    \end{tabular}
}
\end{center}
\vspace{-0.1cm}
\caption{The accuracy comparison on USPS for $N_{l}\in\{50, 10, 5, 1\}$.} \label{tab3}
\vspace{-0.1cm}
\end{table}

\begin{figure}
  \centering
  \begin{subfigure}[b]{0.495\linewidth}
    \includegraphics[trim=0.1cm 0.1cm 0.1cm 0.1cm, width=4.0cm, clip]{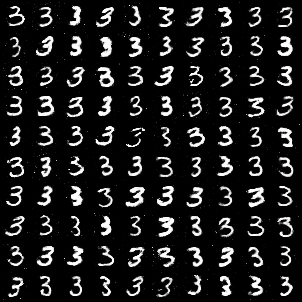}
  \end{subfigure}
  \begin{subfigure}[b]{0.495\linewidth}
    \includegraphics[trim=0.1cm 0.1cm 0.1cm 0.1cm, width=4.0cm, clip]{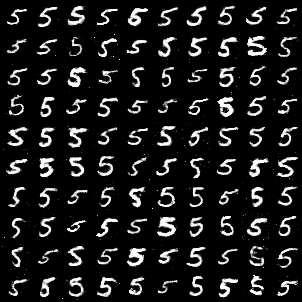}
  \end{subfigure}
  \begin{subfigure}[b]{0.495\linewidth}
    \includegraphics[trim=1.4cm 0.2cm 0.4cm 0.6cm, width=\linewidth, clip]{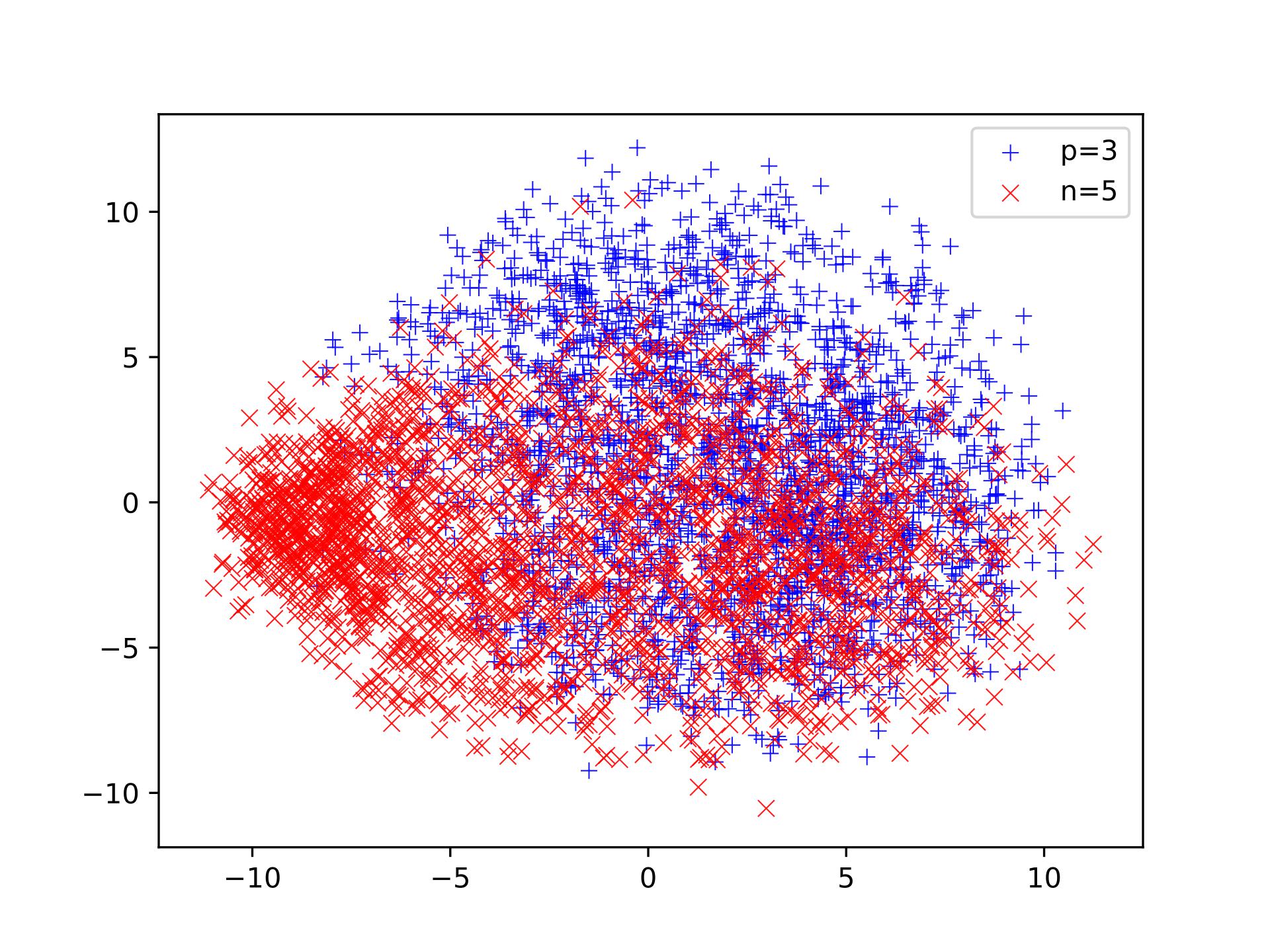}
  \end{subfigure}
  \begin{subfigure}[b]{0.495\linewidth}
    \includegraphics[trim=1.4cm 0.2cm 0.4cm 0.6cm, width=\linewidth, clip]{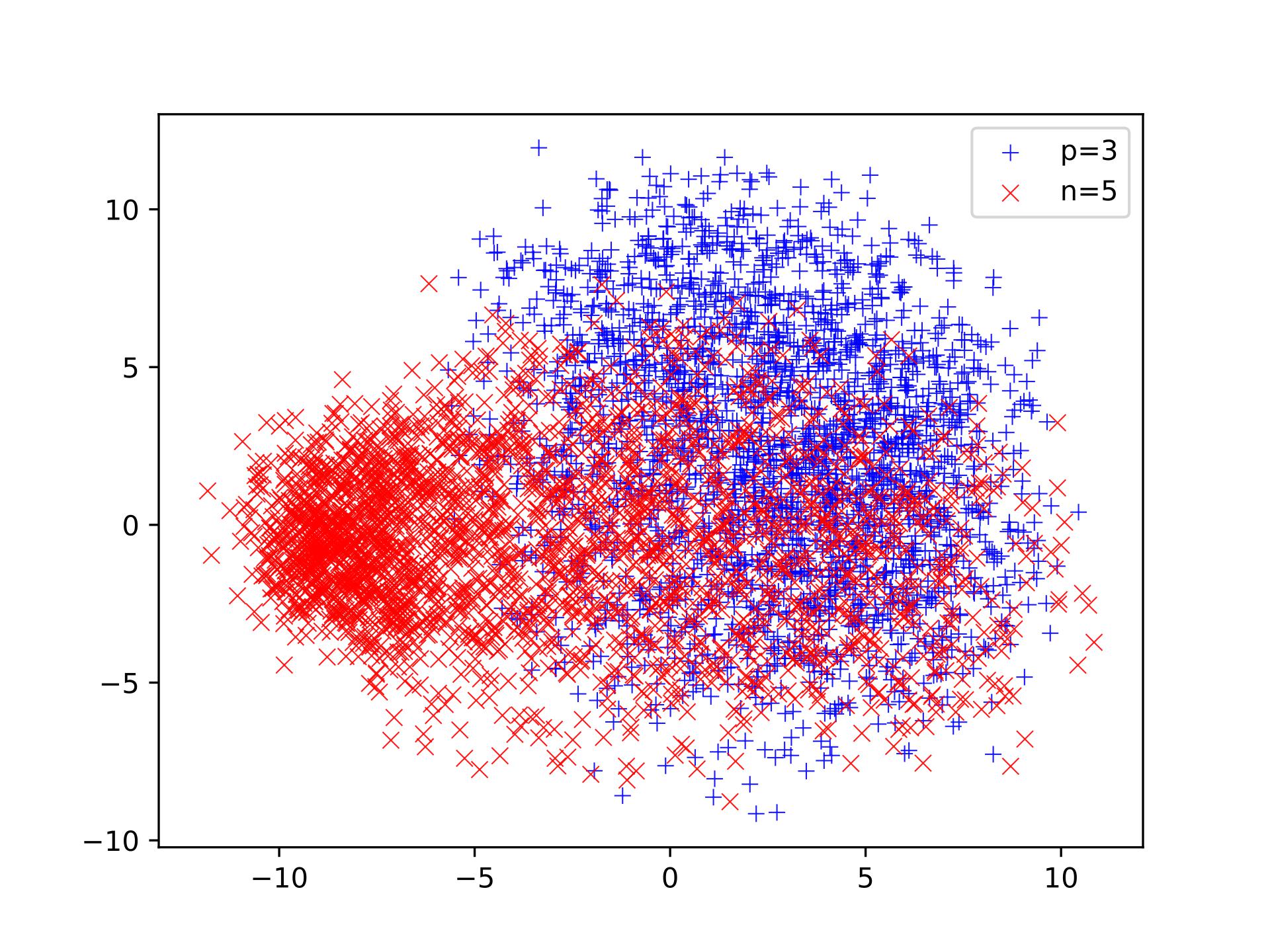}
  \end{subfigure}
  \caption{\small Top$\colon$visualization of positive (left) and negative (right) digits generated using one positive `3' label.  Bottom$\colon$projected distributions of `3' vs `5', with ground truth (left) and generated (right).} \label{fig4}
  \vspace{-0.2cm}
\end{figure}

\section{Discussion}
One key factor to the success of GenPU relies on the capability of underlying GAN in generating diverse samples with high quality standard. Only in this way, the ideal performance could be achieved by training a flexible classifier on those samples. However, it is widely known that the perfect training of the original GAN is quite challenging. GAN suffers from issues of mode collapse and mode oscillation, especially when high-dimensional data distribution has a large number of output modes. For this reason, the similar issue of the original GAN may potentially happen to our GenPU when a lot of output submodes exist. Since it has empirically been shown that the original GAN equipped with JSD inclines to mimic the mode-seeking process towards convergence. Fortunately, our framework is very flexible and generalizable in the sense that it can be established by switching to different underlying GAN variants with more effective distance metrics (i.e., integral probability metric (IPG)) other than JSD (or f-divergence). By doing so, the possible issue of the missing modes can be greatly reduced. For future work, another possible solution is to extend single generator $G_{p}$ ($G_{n}$) to multiple generators $\{G_{p}^{i}\}_{i=1}^{I}$ ($\{G_{n}^{j}\}_{i=1}^{J}$) for the positive (negative) class, also by utilizing the parameter sharing scheme to leverage common information and reduce the computational load.

\bibliographystyle{named}
\bibliography{ijcai18}

\begin{thebibliography}{}

\bibitem[\protect\citeauthoryear{Arjovsky \bgroup \em et al.\egroup
  }{2017}]{arjovsky2017wasserstein}
Martin Arjovsky, Soumith Chintala, and L{\'e}on Bottou.
\newblock Wasserstein gan.
\newblock {\em arXiv preprint arXiv:1701.07875}, 2017.

\bibitem[\protect\citeauthoryear{Denis \bgroup \em et al.\egroup
  }{2005}]{denis2005learning}
Fran\c{c}ois Denis, R{\'e}mi Gilleron, and Fabien Letouzey.
\newblock Learning from positive and unlabeled examples.
\newblock {\em Theor. Comput. Sci.}, 348(1):70--83, 2005.

\bibitem[\protect\citeauthoryear{Denis}{1998}]{denis1998pac}
Fran{\c{c}}ois Denis.
\newblock Pac learning from positive statistical queries.
\newblock In {\em ALT}, volume~98, pages 112--126. Springer, 1998.

\bibitem[\protect\citeauthoryear{du Plessis \bgroup \em et al.\egroup
  }{2014}]{du2014analysis}
Marthinus~C du~Plessis, Gang Niu, and Masashi Sugiyama.
\newblock Analysis of learning from positive and unlabeled data.
\newblock In {\em Advances in neural information processing systems}, pages
  703--711, 2014.

\bibitem[\protect\citeauthoryear{Du~Plessis \bgroup \em et al.\egroup
  }{2015}]{du2015convex}
Marthinus Du~Plessis, Gang Niu, and Masashi Sugiyama.
\newblock Convex formulation for learning from positive and unlabeled data.
\newblock In {\em International Conference on Machine Learning}, pages
  1386--1394, 2015.

\bibitem[\protect\citeauthoryear{Durugkar \bgroup \em et al.\egroup
  }{2016}]{durugkar2016generative}
Ishan Durugkar, Ian Gemp, and Sridhar Mahadevan.
\newblock Generative multi-adversarial networks.
\newblock {\em arXiv preprint arXiv:1611.01673}, 2016.

\bibitem[\protect\citeauthoryear{Elkan and Noto}{2008}]{elkan2008learning}
Charles Elkan and Keith Noto.
\newblock Learning classifiers from only positive and unlabeled data.
\newblock In {\em Proceedings of the 14th ACM SIGKDD international conference
  on Knowledge discovery and data mining}, pages 213--220. ACM, 2008.

\bibitem[\protect\citeauthoryear{Ghosh \bgroup \em et al.\egroup
  }{2017}]{ghosh2017multi}
Arnab Ghosh, Viveka Kulharia, Vinay Namboodiri, Philip~HS Torr, and Puneet~K
  Dokania.
\newblock Multi-agent diverse generative adversarial networks.
\newblock {\em arXiv preprint arXiv:1704.02906}, 2017.

\bibitem[\protect\citeauthoryear{Goodfellow \bgroup \em et al.\egroup
  }{2014}]{goodfellow2014generative}
Ian Goodfellow, Jean Pouget-Abadie, Mehdi Mirza, Bing Xu, David Warde-Farley,
  Sherjil Ozair, Aaron Courville, and Yoshua Bengio.
\newblock Generative adversarial nets.
\newblock In {\em Advances in neural information processing systems}, pages
  2672--2680, 2014.

\bibitem[\protect\citeauthoryear{Hido \bgroup \em et al.\egroup
  }{2008}]{hido2008inlier}
Shohei Hido, Yuta Tsuboi, Hisashi Kashima, Masashi Sugiyama, and Takafumi
  Kanamori.
\newblock Inlier-based outlier detection via direct density ratio estimation.
\newblock In {\em Data Mining, 2008. ICDM'08. Eighth IEEE International
  Conference on}, pages 223--232. IEEE, 2008.

\bibitem[\protect\citeauthoryear{Hoang \bgroup \em et al.\egroup
  }{2017}]{hoang2017multi}
Quan Hoang, Tu~Dinh Nguyen, Trung Le, and Dinh Phung.
\newblock Multi-generator gernerative adversarial nets.
\newblock {\em arXiv preprint arXiv:1708.02556}, 2017.

\bibitem[\protect\citeauthoryear{Jain \bgroup \em et al.\egroup
  }{2016}]{jain2016estimating}
Shantanu Jain, Martha White, and Predrag Radivojac.
\newblock Estimating the class prior and posterior from noisy positives and
  unlabeled data.
\newblock In {\em Advances in Neural Information Processing Systems}, pages
  2693--2701, 2016.

\bibitem[\protect\citeauthoryear{Kiryo \bgroup \em et al.\egroup
  }{2017}]{kiryo2017positive}
Ryuichi Kiryo, Gang Niu, Marthinus C~du Plessis, and Masashi Sugiyama.
\newblock Positive-unlabeled learning with non-negative risk estimator.
\newblock 2017.

\bibitem[\protect\citeauthoryear{LeCun \bgroup \em et al.\egroup
  }{1990}]{lecun1990handwritten}
Yann LeCun, Bernhard~E Boser, John~S Denker, Donnie Henderson, Richard~E
  Howard, Wayne~E Hubbard, and Lawrence~D Jackel.
\newblock Handwritten digit recognition with a back-propagation network.
\newblock In {\em Advances in neural information processing systems}, pages
  396--404, 1990.

\bibitem[\protect\citeauthoryear{LeCun \bgroup \em et al.\egroup
  }{1998}]{lecun1998the}
Yann LeCun, Corinna Cortes, and Christopher~JC Burges.
\newblock The mnist database of handwritten digits.
\newblock 1998.

\bibitem[\protect\citeauthoryear{Lee and Liu}{2003}]{lee2003learning}
Wee~Sun Lee and Bing Liu.
\newblock Learning with positive and unlabeled examples using weighted logistic
  regression.
\newblock In {\em ICML}, volume~3, pages 448--455, 2003.

\bibitem[\protect\citeauthoryear{Li and Liu}{2003}]{li2003learning}
Xiaoli Li and Bing Liu.
\newblock Learning to classify texts using positive and unlabeled data.
\newblock In {\em IJCAI}, volume~3, pages 587--592, 2003.

\bibitem[\protect\citeauthoryear{Li \bgroup \em et al.\egroup
  }{2011}]{li2011positive}
Wenkai Li, Qinghua Guo, and Charles Elkan.
\newblock A positive and unlabeled learning algorithm for one-class
  classification of remote-sensing data.
\newblock {\em IEEE Transactions on Geoscience and Remote Sensing},
  49(2):717--725, 2011.

\bibitem[\protect\citeauthoryear{Liu \bgroup \em et al.\egroup
  }{2002}]{liu2002partially}
Bing Liu, Wee~Sun Lee, Philip~S Yu, and Xiaoli Li.
\newblock Partially supervised classification of text documents.
\newblock In {\em ICML}, volume~2, pages 387--394, 2002.

\bibitem[\protect\citeauthoryear{Liu \bgroup \em et al.\egroup
  }{2003}]{liu2003building}
Bing Liu, Yang Dai, Xiaoli Li, Wee~Sun Lee, and Philip~S Yu.
\newblock Building text classifiers using positive and unlabeled examples.
\newblock In {\em Data Mining, 2003. ICDM 2003. Third IEEE International
  Conference on}, pages 179--186. IEEE, 2003.

\bibitem[\protect\citeauthoryear{Neyshabur \bgroup \em et al.\egroup
  }{2017}]{neyshabur2017stabilizing}
Behnam Neyshabur, Srinadh Bhojanapalli, and Ayan Chakrabarti.
\newblock Stabilizing gan training with multiple random projections.
\newblock {\em arXiv preprint arXiv:1705.07831}, 2017.

\bibitem[\protect\citeauthoryear{Nguyen \bgroup \em et al.\egroup
  }{2017}]{nguyen2017dual}
Tu~Nguyen, Trung Le, Hung Vu, and Dinh Phung.
\newblock Dual discriminator generative adversarial nets.
\newblock In {\em Advances in Neural Information Processing Systems}, pages
  2667--2677, 2017.

\bibitem[\protect\citeauthoryear{Patrini \bgroup \em et al.\egroup
  }{2016}]{patrini2016loss}
Giorgio Patrini, Frank Nielsen, Richard Nock, and Marcello Carioni.
\newblock Loss factorization, weakly supervised learning and label noise
  robustness.
\newblock In {\em International Conference on Machine Learning}, pages
  708--717, 2016.

\bibitem[\protect\citeauthoryear{Salimans \bgroup \em et al.\egroup
  }{2016}]{salimans2016improved}
Tim Salimans, Ian Goodfellow, Wojciech Zaremba, Vicki Cheung, Alec Radford, and
  Xi~Chen.
\newblock Improved techniques for training gans.
\newblock In {\em Advances in Neural Information Processing Systems}, pages
  2234--2242, 2016.

\bibitem[\protect\citeauthoryear{Smola \bgroup \em et al.\egroup
  }{2009}]{smola2009relative}
Alex Smola, Le~Song, and Choon~Hui Teo.
\newblock Relative novelty detection.
\newblock In {\em Artificial Intelligence and Statistics}, pages 536--543,
  2009.

\bibitem[\protect\citeauthoryear{Tolstikhin \bgroup \em et al.\egroup
  }{2017}]{tolstikhin2017adagan}
Ilya Tolstikhin, Sylvain Gelly, Olivier Bousquet, Carl-Johann Simon-Gabriel,
  and Bernhard Sch{\"o}lkopf.
\newblock Adagan: boosting generative models.
\newblock {\em arXiv preprint arXiv:1701.02386}, 2017.

\bibitem[\protect\citeauthoryear{Wang \bgroup \em et al.\egroup
  }{2016}]{wang2016ensembles}
Yaxing Wang, Lichao Zhang, and Joost van~de Weijer.
\newblock Ensembles of generative adversarial networks.
\newblock {\em arXiv preprint arXiv:1612.00991}, 2016.

\bibitem[\protect\citeauthoryear{Ward \bgroup \em et al.\egroup
  }{2009}]{ward2009presence}
Gill Ward, Trevor Hastie, Simon Barry, Jane Elith, and John~R Leathwick.
\newblock Presence-only data and the em algorithm.
\newblock {\em Biometrics}, 65(2):554--563, 2009.

\end{thebibliography}

\end{document}